\documentclass{elsarticle}

\usepackage{amsmath}
  \usepackage{paralist}
  \usepackage{graphics} 
  \usepackage{epsfig} 
\usepackage{graphicx}  \usepackage{epstopdf}
 \usepackage[colorlinks=true]{hyperref}
\hypersetup{urlcolor=blue, citecolor=red}
\usepackage{multirow}

\usepackage[frozencache,newfloat,cachedir=mintedcache]{minted}

\usepackage{caption}

\newenvironment{code}{\captionsetup{type=listing}}{}
\SetupFloatingEnvironment{listing}{name=Code Snippet}

\newcommand{\eps}{\varepsilon}
\usepackage{algorithm,algorithmicx,algpseudocode}
\usepackage{amsthm}
\usepackage{amssymb}
\usepackage{booktabs}

\DeclareMathOperator{\id}{id}
\DeclareMathOperator{\Real}{Re}

\renewcommand{\Re}{\Real}

\DeclareMathOperator*{\argmax}{argmax}
\DeclareMathOperator*{\argmin}{argmin}
\DeclareMathOperator{\diag}{diag}
\DeclareMathOperator{\prox}{prox}

\DeclareMathOperator{\spec}{spectrum}
\renewcommand{\geq}{\geqslant}
\renewcommand{\leq}{\leqslant}
\newcommand{\R}{\mathbf R}

\newcommand{\N}{\mathbf N}
\newcommand{\C}{\mathbf C}
\newcommand{\expectation}{\mathbf E}
\usepackage{cleveref}
\newtheorem{theorem}{Theorem}[section]
\newtheorem{example}{Example}[section]

\newtheorem{lemma}[theorem]{Lemma}

\theoremstyle{definition}
\newtheorem{definition}[theorem]{Definition}
\newtheorem{remark}{Remark}
\begin{document}
\begin{frontmatter}
	% Place the running head in [], and the full title of the paper in {}.
	\title{Designing Stable Neural Networks using Convex Analysis and ODEs} % Only the first word and proper nouns should be capitalized
	\author[1]{Ferdia Sherry\corref{cor1}}
	\ead{fs436@cam.ac.uk}
	\author[2]{Elena Celledoni}
	\ead{elena.celledoni@ntnu.no}
	\author[3]{Matthias J.\ Ehrhardt}
	\ead{me549@bath.ac.uk}
	\author[2]{Davide Murari}
	\ead{davide.murari@ntnu.no}
	\author[2]{Brynjulf Owren}
	\ead{brynjulf.owren@ntnu.no}
	\author[1]{Carola-Bibiane Sch\"onlieb}
	\ead{cbs31@cam.ac.uk}

	\affiliation[1]{organization={Department of Applied Mathematics and Theoretical Physics, University of Cambridge},
		addressline={Wilberforce Road},
		city={Cambridge},
		postcode={CB3 0WA},
		country={UK}}

	\affiliation[2]{organization={Department of Mathematical Sciences, Norwegian University of Science and Technology},
		addressline={Institutt for matematiske fag},
		city={Trondheim},
		postcode={7491},
		country={Norway}}

	\affiliation[3]{organization={Department of Mathematical Sciences, University of Bath},
		addressline={6 West},
		city={Bath},
		postcode={BA2 7JU},
		country={UK}}

	\cortext[cor1]{Corresponding author}
	% Please use `and' to connect the last two names if applicable
	% List full names if possible. If all authors' full names will not fit, use FirstNameInitial. MiddleNameInitial. LastName, only last names, or full names of first few authors, et al

	\begin{abstract}
		Motivated by classical work on the numerical integration of ordinary differential equations we present a ResNet-styled neural network architecture that encodes non-expansive (1-Lipschitz) operators, as long as the spectral norms of the weights are appropriately constrained. This is to be contrasted with the ordinary ResNet architecture which, even if the spectral norms of the weights are constrained, has a Lipschitz constant that, in the worst case, grows exponentially with the depth of the network. Further analysis of the proposed architecture shows that the spectral norms of the weights can be further constrained to ensure that the network is an averaged operator, making it a natural candidate for a learned denoiser in Plug-and-Play algorithms. Using a novel adaptive way of enforcing the spectral norm constraints, we show that, even with these constraints, it is possible to train performant networks. The proposed architecture is applied to the problem of adversarially robust image classification, to image denoising, and finally to the inverse problem of deblurring.
%%% Local Variables:
%%% mode: latex
%%% TeX-master: "main"
%%% End:

	\end{abstract}


\end{frontmatter}

\section{Introduction}
The desire to impose Lipschitz conditions on neural networks has come to the forefront in a number of tasks in recent years, especially because there have been serious concerns about the stability of neural networks ever since it was shown that high-performance image classifiers may suffer from adversarial examples~\cite{goodfellow_explaining_2015}. These issues need to be satisfactorily resolved before deep learning methods can be considered suitable for application in safety-critical systems. Another important application of Lipschitz neural networks can be found in generative modelling, in particular in models such as Wasserstein generative adversarial networks (GANs)~\cite{arjovsky_wasserstein_2017}. In these models, the aim is to minimise the Wasserstein distance between the output of a generator neural network and some target distribution:
\begin{equation}
	\label{prob:learning_generator}
	\min_{\Psi} W_1(\Psi\#\mu_{\mathrm{latent}}, \mu_{\mathrm{true}}),
\end{equation}
where $W_1$ is the Wasserstein metric, $\mu_{\mathrm{latent}}$ is a (simple) distribution of latent variables $Z\in\mathcal Z$, $\Psi\# \mu_{\mathrm{latent}}$ is its pushforward by the generator neural network $\Psi:\mathcal Z\to\mathcal X$ and $\mu_{\mathrm{true}}$ is the target distribution of $X\in \mathcal X$. Appealing to the Kantorovich-Rubinstein duality, we know that
\[W_1(\mu, \nu) = \sup_{f:\mathcal X\to \R,\mathrm{1-Lipschitz}} \expectation_{X\sim \mu}[f(X)] - \expectation_{Y\sim \nu}[f(Y)],\]
where $f$ is usually called the critic. With this result, \Cref{prob:learning_generator} becomes the following saddle point problem:
\[\min_{\Psi}\sup_{f:\mathcal X\to \R, \mathrm{1-Lipschitz}}\expectation_{Z\sim \mu_{\mathrm{latent}}} [f(\Psi(Z))] - \expectation_{X\sim \mu_{\mathrm{true}}} [f(X)]. \]
To solve this problem, we are required to flexibly parametrise 1-Lipschitz critic functions $f:\mathcal X\to \R$.

Another application in which neural network instabilities may cause considerable problems is the application of deep learning to inverse problems. A prototypical approach to doing this is the so-called Plug-and-Play (PnP) approach~\cite{venkatakrishnanPlug2013, chanAlgorithm2016,chanPlug2016, sreehariPlug2016}, in which some parts of an iterative optimisation algorithm (usually the denoising steps) for a variational reconstruction problem are replaced by a different component, which may be learned separately. In this setting, we generally assume that we have a forward operator $K$ that maps (uncorrupted) images to (uncorrupted) measurements, e.g.\ a convolution in the case of a deblurring problem or the Radon transform in the case of computed tomography. The measurements are furthermore corrupted by noise to give noisy measurements $y$ and the goal is to recover the original image as well as possible from $y$, which is made more complicated by the usual ill-posedness of this inversion. The variational approach to regularising inverse problems overcomes this ill-posedness by minimising an objective function that balances data mismatch with fit to prior knowledge. Mismatch between a recovered image and the noisy measurements $y$ is measured in the form of a data discrepancy functional $E_y$, which often takes a least-squares form: $E_y(x) = \|Kx - y\|^2$. Rather than use an explicit regularisation functional as in the usual variational approach, the PnP approach uses a denoiser $\Phi$ to encode information about the fit to prior knowledge. A simple example of this approach is shown in~\Cref{alg:pnp_fb}, using a PnP method based on the proximal gradient method. The initial estimate $x^0$ may be set in various ways: we may use a constant initialisation, or apply a (regularised) pseudoinverse of $K$ to $y$ to obtain $x^0$. Further parameters of the method are the number of iterations $\texttt{n\_iter}$ and the step size $\tau > 0$ for the gradient steps on the data discrepancy. In this setting, the denoiser $\Phi$ may be a denoiser trained on natural images.
\begin{algorithm}[ht!]
	\caption{Plug-and-Play proximal gradient method}
	\label{alg:pnp_fb}
	\begin{algorithmic}
		\State \textbf{inputs:} noisy measurements $y$, initial estimate $x^0$, denoiser $\Phi$, step size $\tau$, number of iterations \texttt{n\_iter}
		\State $x\leftarrow x^0$
		\For{$i\leftarrow 1,\ldots, \texttt{n\_iter}$}
		\State $x\leftarrow \Phi(x - \tau \nabla E_y(x))$
		\EndFor\\
		\Return $x$
	\end{algorithmic}

\end{algorithm}

With such algorithms, we may run into divergent behaviour if we do not restrict $\Phi$ appropriately~\cite{sommerhoffEnergyDissipationPlugandPlay2019}, but there is recent work showing that the iterative method will converge as long as certain Lipschitz conditions are imposed on the denoiser $\Phi$~\cite{ryu_plug-and-play_2019,hertrichConvolutionalProximalNeural2021}.

Lipschitz continuity is a standard way to quantify the stability of a function. Let us recall its definition and some associated properties: a function $f:\mathcal X\to \mathcal Y$ between metric spaces $(\mathcal X, d_{\mathcal X})$ and $(\mathcal Y, d_{\mathcal Y})$ is said to be $L$-Lipschitz for some $L\geq 0$ if $d_\mathcal Y(f(x_1), f(x_2))\leq L d_\mathcal X(x_1, x_2)$ for all $x_1, x_2\in \mathcal X$. This notion of stability plays well with the compositional nature of neural networks: if $f_1 :\mathcal X\to \mathcal Y$ and $f_2:\mathcal Y\to \mathcal Z$ are $L_1$-Lipschitz and $L_2$-Lipschitz respectively, their composition $f_2 \circ f_1$ is $(L_1\cdot L_2)$-Lipschitz. If $\mathcal X$ and $\mathcal Y$ are in fact normed spaces, we can furthermore see (by definition) that any bounded linear operator $A:\mathcal X\to \mathcal Y$ is $\|A\|$-Lipschitz, where the norm is the operator norm. In particular, an ordinary feedforward neural network $\Psi(x) = A^K\sigma(b^{K-1} + A^{K-1}\sigma(\ldots +A^2\sigma(b^1 + A^1x)))$ with a $1$-Lipschitz activation function $\sigma$ and learnable linear operators $A^1, \ldots, A^K$ and biases $b^1, \ldots, b^K$ is $L$-Lipschitz, where $L = \prod_{i=1}^K \|A^i\|$. This naturally gives rise to the idea of spectral normalisation: if an ordinary feedforward neural network with a given Lipschitz constant $L$ is required for a specific application, this can be achieved by appropriately normalising the linear operators, as applied to GANs in~\cite{miyato_spectral_2018}. It is worth remarking here that we are referring to any $L$ that satisfies the defining inequality for Lipschitz continuity of $f$ as a Lipschitz constant of $f$; often the term Lipschitz constant is used instead to refer only to the infimum of such $L$, which defines a seminorm on vector spaces of Lipschitz functions. We will refer to this infimum as the optimal Lipschitz constant of $f$, and note that the statements about the composition of Lipschitz functions, when framed in terms of optimal Lipschitz constants, only give upper bounds in general.

In this work we are focused on the case where $\mathcal X$ and $\mathcal Y$ are equal to each other, as is the case in many image-to-image tasks. Residual networks (ResNets)~\cite{he_delving_2015,he_deep_2016} have proven to be an extremely successful neural network meta-architecture in this setting, especially when identity skip connections are used~\cite{heIdentityMappingsDeep2016}, as we will assume in the rest of this paper: a ResNet parametrises a neural network by $\Psi = (\id + \Psi^K) \circ\ldots \circ (\id + \Psi^1)$, where each $\Psi^i$ is a small neural network. Without further constraints, the Lipschitz continuity of such a network may be badly behaved as the depth increases: even if we control each $\Psi^i$ to be $\eps$-Lipschitz for some small $\eps >0$, in the worst case we can not guarantee anything better than that $\id +\Psi^i$ is $(1+\eps)$-Lipschitz, and that the composition $\Psi$ is $L$-Lipschitz with $L=(1 + \eps)^K $, which grows exponentially as $K\to\infty$. Nevertheless, in this work we will show that is possible to design ResNets that are provably non-expansive ($1$-Lipschitz), while retaining high performance on tasks of interest. The design of these networks is based on the fact that gradient flows in convex potentials are non-expansive, and the fact that one-hidden-layer neural networks with appropriately tied weights are gradients of learnable convex potentials. Hence, parametrising the $\Psi^i$ in the ResNet above in this way, we might expect to obtain non-expansive neural networks, being the composition of discretisations of gradient flows in convex potentials. It is crucial, however, to be careful about how we implement this: in general, discretisations of non-expansive flows are not guaranteed to be non-expansive. Using the concept of circle contractivity, from the study of numerical methods for ODEs, we show that it is still possible to ensure the non-expansiveness of the proposed networks if we control the operator norm of their weights.
\subsection{Related topics}
\subsubsection*{Lipschitz neural networks}
As mentioned above, within the deep learning community there have been a number of drivers for research into neural networks with controlled Lipschitz constants, such as the desire to increase robustness to adversarial examples~\cite{tsuzukuLipschitzMarginTrainingScalable2018}, the necessity to model the critic in a Wasserstein GAN as a $1$-Lipschitz function~\cite{arjovsky_wasserstein_2017}, and their use in parametrising learnable diffeomorphisms~\cite{celledoni2022deep}. Spectral normalisation~\cite{miyato_spectral_2018} has become a standard approach to constraining the Lipschitz constant of an ordinary feedforward neural network. This approach ensures that the optimal (smallest) Lipschitz constant of a neural network is upper bounded. It is known to be computationally hard to estimate the true optimal Lipschitz constant~\cite{virmaux_lipschitz_2018} of a neural network, which has prompted further research into refining Lipschitz neural network architectures.
\subsubsection*{Methods based on continuous dynamical systems}
Applied mathematicians and physicists have long studied continuous dynamical systems in the form of ODEs and PDEs, giving rise to an extensive body of research on the structural properties of such systems. More recently, insights from these topics have been used to design neural network architectures which share similar structural properties~\cite{ruthotto_deep_2020,chang_reversible_2018,celledoni2022dynamical}. The adjoint method for computing gradients has gained widespread use in the deep learning community, after it was shown in the Neural ODEs paper~\cite{chen_neural_2018} that it is possible to parametrise the vector field defining an ODE by a neural network and differentiate through the flow to learn the vector field. This work has spawned a plethora of works that use learnable continuous dynamical systems.
\subsubsection*{Convex analysis and monotone operator theory}
There is a recent line of work investigating the connections between existing deep learning practice and the topics of convex analysis and monotone operator theory. In particular, many of the standard activation functions that are used in neural networks are averaged (in the sense that we define in~\Cref{sec:euler_averaged}), and further analysis enables one to use this insight to design neural networks that are averaged~\cite{combettes_deep_2020,pesquet_learning_2020,hasannasab_parseval_2020,hertrichConvolutionalProximalNeural2021}.
\subsection{Our contributions}
We describe and analyse a family of ResNet-styled neural network architectures that are guaranteed to be non-expansive. The effect of these neural networks on input vectors can be thought of as sequentially composing parts of (discretisations of) gradient flows along learnable convex potentials. We show that it is only necessary to control the operator norms of the learnable linear operators contained in these networks to ensure their non-expansiveness. This task is easily achieved in practice using power iteration.

The most basic such network takes the simple form described in~\Cref{alg:euler_lipschitz}. For this network, we use convex analysis techniques to show that more fine-grained control of the learnable linear operators ensures that each layer of the network is averaged, and as a result that the overall network is averaged.

We demonstrate the use of the proposed architectures on various tasks. The first problem that we consider is image classification, with a special focus on adversarial robustness on the CIFAR10 dataset. Here, non-expansiveness of the proposed networks is shown to be an effective inductive bias to increase the robustness of ResNet-styled networks. Secondly, we study the natural application of the proposed architectures to an image denoising task on the BSDS500 dataset, focusing on the influence of various tunable aspects in the architectures for this problem, and comparing our approach to standard approaches to the denoising task. Finally, we apply one of the learned denoisers, which is provably an averaged operator, to regularise the solution of an ill-posed inverse problem using PnP methods as described above, with a provable convergence guarantee.
%%% Local Variables:
%%% mode: latex
%%% TeX-master: "main"
%%% End:

\label{sec:introduction}

\section{Methods}
\label{sec:nonexpansive_odes}
Suppose that $f:\R\times \R^n \to \R^n$ is a time-dependent vector field and consider the ODE given by the flow along this vector field:
\begin{equation}
	\label{eq:ode}
	\dot z(t) = f(t, z(t)).
\end{equation}
Assuming existence and uniqueness of the solutions to the ODE, we can define the flow map $\Psi : [0,\infty)\ \times \R^n \to \R^n$ by $\Psi(t, x) =z(t)$, where $z$ solves~\Cref{eq:ode} with the initial condition $z(0) = x$. Since the vector fields that we will consider are (globally) Lipschitz continuous, global existence and uniqueness is not an issue by the Picard-Lindel\"of theorem~\cite{teschl2012}. It is natural to ask when this flow map is non-expansive, in the sense that $\|\Psi(t, x) - \Psi(t, y)\|\leq \|x - y\|$ for all $t,x,y$ (with $\|x\| := \sqrt{\langle x, x\rangle}$). Letting $t\to 0$, we see that it is necessary that
\begin{equation}\langle f(t, x) - f(t, y), x-  y\rangle \leq 0,
	\label{ineq:vec_field_monotone}
\end{equation}
and conversely, if this condition holds the flow map is non-expansive since
\begin{equation}
	\label{ineq:nonexpansive_flow}
	\frac{\mathrm d}{\mathrm d t} \|\Psi(t ,x) - \Psi(t, y)\|^2 = 2 \langle f(t, x(t)) - f(t, y(t)), x(t) - y(t)\rangle \leq 0.
\end{equation}
In practice, most ODEs of interest are not explicitly solvable and it is necessary to resort to numerical methods to approximate the flow map. A very well-studied class of such numerical integrators is the class of Runge--Kutta methods, which can be defined as follows:
\begin{definition}[Runge--Kutta method]
	\label{def:runge_kutta}
	If $m\in\N$ is a positive integer, an $m$-stage Runge--Kutta (RK) method is characterised by a matrix $\mathcal A\in \R^{m\times m}$ and two vectors $b, c\in \R^m$. For a step size $h>0$, the RK method approximates the step from $y=\Psi(t, x)$ to $\Psi(t + h, x)$ as follows:
	\[\Phi_h(t, y, f) = y + h\sum\limits_{i=1}^m b_i f(t + c_i h, Y_i),\]
	where $Y=(Y_1,\ldots, Y_m)$ (the set of so-called stages of the method) solves the non-linear system of equations
	\[Y_i =y + h \sum\limits_{j=1}^m \mathcal A_{ij} f(t + c_j h, Y_j)\quad\text{for}\quad i=1,\ldots, m.\]
	If $\mathcal A$ is strictly lower triangular, these equations are solvable in a single pass and the method is called explicit. Otherwise, the method is called an implicit method. To ensure that the method is at least of order one, we require that $\sum_{j=1}^mb_j=1$. Furthermore, an RK method will usually satisfy $c_i = \sum_{j=1}^m \mathcal A_{i,j}$. If $c_1,\ldots,c_m$ are distinct, the method is called non-confluent, whereas it is confluent otherwise.
\end{definition}
The most basic, and most well-known, example of an RK method is the forward Euler method, which is an explicit method. Given a time, $t$, and an approximation of the trajectory at that time, $y$, the Euler method simply approximates the trajectory at the next time step, $t + h$, as $y + h f(t, y)$. By a Taylor expansion, this approximation is easily seen to incur a (local) error of order $h^2$ as $h\to 0$. The forward Euler method is a 1-stage RK method, while general RK methods with $m\geq 1$ stages make use of the stages to approximate multiple intermediate values, enabling higher orders of approximation to be achieved. When an RK method incurs local errors of order $h^{p+1}$ as $h\to 0$, we say that the method has order $p$. For more general methods, Taylor expansions can be made of the true flow map and of the approximate flow map given by the method. These Taylor expansions can be compared term by term to derive the order of approximation. For more details on RK methods and conditions on their orders of approximation, the reader is recommended to consult~\cite{hairerSolvingOrdinaryDifferential1993}.

The form of the RK method (especially the forward Euler method) given in \cref{def:runge_kutta} is reminiscent of the form of the ResNet. Indeed, recall that a ResNet $\Psi$ can be written as $\Psi = (\id + \Psi^K) \circ \ldots \circ (\id + \Psi^1)$, with each $\Psi^i$ a smaller neural network. This is immediately recognised as a composition of $K$ steps of the forward Euler method, with (arbitrary) step sizes $h^i > 0$ and vector fields $f^i = \Psi_i / h^i$ (varying across the steps). Once this connection has been established, it is natural to split the problem of designing a neural network architecture into two stages: the first stage being the design of parametrised vector fields, the continuous flows of which have desirable structural properties, followed by the application of a suitable RK method to discretise the continuous flows in such a way that the desirable structural properties are preserved.

Since we aim to design neural networks that encode non-expansive operators, it is of particular interest to know whether a given numerical integrator preserves the non-expansiveness of a continuous flow for which~\Cref{ineq:nonexpansive_flow} holds. This property of a numerical integrator is called BN-stability and has been studied in detail for RK methods in~\cite{burrage_stability_1979}; for these methods, BN-stability is equivalent to algebraic stability, which is defined by a simple algebraic condition on the coefficients on the method. For methods that are non-confluent, these conditions are also equivalent to the condition of AN-stability (which is a priori a simpler condition). Although the names of these concepts of stability may seem cryptic at a first glance, they can be understood as follows: A-stability and AN-stability are concerned with the preservation of the non-expansiveness of flows of scalar linear vector fields under discretisation, for the autonomous (time-independent) and non-autonomous (time-dependent) cases respectively. On the other hand, B-stability and BN-stability ask that non-expansiveness of general non-expansive flows (i.e.\ with vector fields satisfying the monotonicity condition in \Cref{ineq:vec_field_monotone}) is preserved under discretisation, again for the autonomous and non-autonomous cases respectively. A comprehensive overview of stability properties for RK methods is given in~\cite{hairer_solving_1996, dekkerStabilityRungeKuttaMethods1984}. It is well known (see for instance~\cite{nevanlinna_nonexistence_1974}) however that no explicit RK method can satisfy such an unconditional stability condition. Nevertheless, it was shown in~\cite{dahlquist_generalized_1979} that a conditional stability result can be established for certain explicit RK methods as long as~\Cref{ineq:nonexpansive_flow} is replaced by an alternative that has the effect of controlling the stiffness of the ODE\@. To state this result, we require the definition of the circle contractivity property of an RK method.
\begin{definition}[Circle contractivity]
	Suppose that $\mathcal A\in \R^{m\times m}$ and $b,c\in\R^m$ are the matrix and vectors characterising an RK method as in~\Cref{def:runge_kutta}. We say that this RK method satisfies the $r$-circle contractivity condition for a given $r\in\R\cup\{\infty\}$ if $|K(\zeta)|\leq 1$ for all $\zeta \in D(r)^m$. Here, the function $K:\C^m \to \C$ is defined~\footnote{To understand the definition of $K$, it can be thought of as the action of the RK method on a scalar non-autonomous linear ODE.} as:
	\[K(\zeta) =  1 + b^\top\diag(\zeta)(\id - \mathcal A \diag(\zeta))^{-1}\mathbf 1,\]
	and $D(r)$ is a generalised disk:
	\[D(r) =
		\begin{cases}
			\{z\in \C | |z + r| \leq r\}\quad  & \text{when}\quad r \geq 0,   \\
			\{z\in \C | \Re(z) \leq 0\}\quad   & \text{when}\quad r = \infty, \\
			\{z\in \C | |z + r| \geq -r\}\quad & \text{when}\quad r < 0.
		\end{cases}
	\]
	In the above definition, $\Re(z)$ takes the real part of a complex number $z$.
\end{definition}
\begin{example}
	Recall the forward Euler method, which was introduced above. In the notation of~\Cref{def:runge_kutta}, we have $m=1$, $\mathcal A=0$ and $b=1$, so $K(z) = 1 + z$. We conclude that the forward Euler method is $1$-circle contractive.
\end{example}
\begin{remark}
	\label{remark:computing_r}
	It is straightforward to compute the optimal (in the sense that it gives the largest generalised disk) $r$ for which a given RK method is $r$-circle contractive if we know $\mathcal A$ and $b$: if we define the symmetric matrix $Q=\diag(b)\mathcal A + \mathcal A^\top\diag(b) -b b^\top$, theorem 3.1 from~\cite{dahlquist_generalized_1979} tells us that $r=-1/\rho$, where $\rho$ is the largest number such that $w^\top Qw \geq \rho w^\top\diag(b)w$ for all $w\in \R^m$. Hence, if we can solve the generalised eigenvalue problem $Qv = \lambda\diag(b) v$, we know that the minimal eigenvalue gives the desired $\rho$.
\end{remark}
With this definition, it is now possible to state the conditional stability result that extends to certain explicit methods:
\begin{theorem}[Theorem 4.1 from~\cite{dahlquist_generalized_1979}]
	\label{thm:circle_contr}
	Suppose that $\Phi_h$ is an RK method satisfying the $r$-circle contractivity condition, and that $f$ satisfies the monotonicity condition
	\begin{equation}
		\label{eq:monotonicity_f}
		\langle f(t, y) - f(t, z), y - z \rangle \leq -\nu \|f(t, y) - f(t, z)\|^2.
	\end{equation}
	Then, if $r\neq\infty$ and $h/r \leq 2\nu$, or if $r = \infty$ and $\nu \geq 0$,
	\[\|\Phi_h (t, y, f) - \Phi_h(t, z, f)\|\leq \|y - z\|.\]
\end{theorem}
The idea of using this result to design non-expansive neural networks was recently discussed in~\cite{celledoniStructurepreservingDeepLearning2021}, though in this work no indication was given of how the vector fields should be parametrised. The monotonicity condition given by~\Cref{eq:monotonicity_f} is reminiscent of the property of co-coercivity, known mainly from the theory of convex optimisation for its use in the Baillon-Haddad theorem:
\begin{theorem}[Corollary 18.16 from~\cite{bauschke_convex_2011}]
	\label{thm:bh}
	Suppose that $\phi:\mathcal X\to \R$ is a Fr\'echet-differentiable convex function on a Hilbert space $\mathcal X$. Then $\phi$ is $L$-smooth for some $L\geq 0$ (equivalently, $\nabla \phi$ is $L$-Lipschitz), meaning that
	\[\phi(y)\leq \phi(x) + \langle \nabla \phi(x), y- x\rangle  + \frac{L}{2} \|y - x\|^2,\]
	if and only if $\nabla \phi$ is $1/L$-co-coercive, meaning that
	\[\langle \nabla \phi(y) - \nabla \phi(x), y - x\rangle \geq \frac{1}{L} \|\nabla \phi(y) - \nabla \phi(x)\|^2.\]
\end{theorem}
Indeed, if $f(t, x) = -\nabla \phi(x)$ for a $1/\nu$-smooth convex potential $\phi:\R^n \to \R$ (so that we have a gradient flow of a smooth convex potential), then~\Cref{eq:monotonicity_f} is satisfied. This connection has recently been used to demonstrate in~\cite{sanz_serna_contractivity_2020} that there is an explicit RK method for which the circle contractivity disk degenerates to a point, by constructing a smooth convex potential for which the non-expansiveness of the flow map is not preserved.

For the purpose of using this observation and~\Cref{thm:circle_contr} to design non-expansive neural networks, note the following result:
\begin{lemma}
	\label{lemma:arch_stiffness}
	Suppose that $\sigma :\R\to \R$ is a non-decreasing $L$-Lipschitz activation function, $A\in\R^{n\times k}$ is a matrix and $b\in \R^n$ is a bias vector. The vector field $f_{A, b}(t, x) = - A^\top\sigma(Ax + b)$ (where $\sigma$ is applied separately to each component) satisfies~\Cref{eq:monotonicity_f} with $\nu =1/(\|A\|^2 L)$.
\end{lemma}
\begin{proof}
	Since $\sigma$ is non-decreasing and $L$-Lipschitz, the function $\psi: \R\to\R$ given by $\psi(t) = \int_0^t \sigma(s)\,\mathrm ds$ is convex and $L$-smooth. Hence, $\phi:\R^n \to \R$ given by
	\[\phi(x) = \sum_{i=1}^n \psi(x_i)\]
	is convex and $L$-smooth. The functional $x \mapsto\phi(Ax + b)$ is convex and by the chain rule it has gradient equal to $-f_{A, b}$ and it is $\|A\|^2L$-smooth. By the comments preceding this lemma, the vector field $f_{A,b}$ satisfies~\Cref{eq:monotonicity_f} with $\nu =1/( \|A\|^2 L)$.
\end{proof}
By the previous observations, we can propose a natural non-expansive neural network architecture as follows. Assume that we have a budget for the number of blocks, \texttt{n\_blocks}. Given an $r>0$ such that we have an $r$-circle contractive RK method $\Phi_h$ and an $L$-Lipschitz increasing activation function $\sigma$, consider linear operators $A^1,\ldots, A^\texttt{n\_blocks}$, biases $b^1,\ldots, b^\texttt{n\_blocks}$ and stepsizes $h^1, \ldots, h^\texttt{n\_blocks}$ and define the operator $\Xi$ by
\[\Xi = \Xi^\texttt{n\_blocks}\circ\ldots\circ \Xi^1,\]
where $\Xi^i(x) = \Phi_{h^i}(0, x, f_{A^i, b^i})$ is one numerical integration step along the vector field $f_{A^i, b^i}$ as defined in~\Cref{lemma:arch_stiffness}. \Cref{lemma:arch_stiffness} and~\Cref{thm:circle_contr} ensure that $\Xi$ is non-expansive as long as
\[h^i \|A^i\|^2L \leq 2r.\]
There are various ways in which this bound can be maintained during training, and the power method can be used to compute the required operator norm: as an example, it is possible to alternate gradient update steps of an optimiser with steps that scale the operators down to satisfy the bounds that are violated after the gradient update. An alternative method, which is more faithful to the dynamical systems interpretation, similarly keeps track of the operator norms of the weights but splits the interval into multiple smaller time steps to guarantee the bound when the gradient updates cause a violation of the bound. This approach, the adaptive approach, is the one that we will take in~\Cref{sec:experiments}. For any explicit RK method, the corresponding neural network $\Xi$ is a residual network. For the forward Euler method, the network takes the particularly simple form shown in~\Cref{alg:euler_lipschitz}.

\begin{algorithm}[!htb]
	\caption{Forward Euler method for non-expansive ODE networks}
	\label{alg:euler_lipschitz}
	\begin{algorithmic}
		\State \textbf{input:} vector $x$
		\State \textbf{parameters:} step sizes $h^i > 0 $, linear operators $A^1,\ldots, A^\texttt{n\_blocks}$ satisfying $h^i\|A^i\|^2L\leq 1$ for $i=1,\ldots,\texttt{n\_blocks}$, and biases $b^1, \ldots, b^\texttt{n\_blocks}$
		\State $z^0 \leftarrow x$
		\For{$i\leftarrow 1, \ldots, \texttt{n\_blocks}$}
		\State $z^i \leftarrow z^{i - 1} - h^i (A^i)^\top\sigma(A^i z^{i - 1} + b^i)$
		\EndFor\\
		\Return $\Xi(x) = z^\texttt{n\_blocks}$
	\end{algorithmic}
\end{algorithm}

As mentioned before, we are focused in this chapter on explicit RK methods since they do not require the solution of a (potentially difficult) non-linear equation at each step. It may be interesting to note, however, what can happen when an implicit numerical method is used, such as the backward Euler method. In that case, each update step in~\Cref{alg:euler_lipschitz} needs to be replaced by solving the equation
\[z^i =z^{i-1} - h^i f_{A^i, b}(z^{i})= z^{i - 1} - h^i (A^i)^\top \sigma(A^i z^i + b^{i}).\]
Recalling from the proof of~\Cref{lemma:arch_stiffness} that $-f_{A,b}$ is the gradient of a convex functional $\phi(A\cdot + b)$, this shows that the update step is given by
\[z^i = (\id + h^i\nabla \phi(A^i\cdot + b^i))^{-1}(z^{i-1})=: \prox_{h^i \phi(A^i\cdot + b^i)}(z^{i-1}),\]
which is the defining equation of the proximal operator~\cite{moreau_proprietes_1963}, a mathematical object that has been studied in great detail in the field of convex analysis. Whether considering it from the ODE viewpoint (the backward Euler method is BN-stable~\cite{burrage_stability_1979}) or from the convex analysis and monotone operator theory viewpoint (proximal operators are non-expansive as the resolvents of monotone operators~\cite[Chapter 23]{bauschke_convex_2011}), proximal operators $\prox_{h \phi(A\cdot + b)}$ are well-defined and non-expansive regardless of the step size $h>0$ and the smoothness of $\phi(A\cdot + b)$. This unconditional stability comes at a cost, though: for general $A$, computing the proximal operator of $\prox_{h \phi(A\cdot + b)}$ is not easy (and becomes more difficult as the condition number of $A$ increases). This issue can be overcome by restricting $A$ to certain special sets of operators (for instance satisfying certain orthogonality properties), in which case the proximal operator may be explicitly computable. This approach is similar to the one taken in~\cite{hasannasab_parseval_2020,hertrichConvolutionalProximalNeural2021} though note that it may be difficult to enforce these constraints on convolution-type linear operators. On the other hand, the operator norm constraints that we are required to enforce with explicit numerical integration methods can be easily controlled using power iteration~\cite{golub_eigenvalue_2000}; all we need is the ability to apply the operator and its adjoint to test vectors.
\subsection{A more detailed look at the architecture for the forward Euler method}
\label{sec:euler_averaged}
When the numerical integrator used is the forward Euler method, as described in~\Cref{alg:euler_lipschitz}, straightforward computations can be used to establish the same results guaranteed by the machinery of~\Cref{thm:circle_contr}, and some more nuanced results. Indeed, it is possible to choose the stepsizes in such a way that the resulting neural network is not just non-expansive, but in fact is also averaged:
\begin{definition}[Definition 4.23 from~\cite{bauschke_convex_2011}]
	\label{def:averaged}
	Suppose that $A:\mathcal X\to \mathcal X$ is an operator mapping a Hilbert space $\mathcal X$ into itself and that $\alpha\in(0, 1)$. We call $A$ an $\alpha$-averaged operator if there is a non-expansive $T:\mathcal X\to \mathcal X$ such that $A = (1-\alpha) \id + \alpha T$. We may also leave $\alpha$ unspecified, in which case we just call $A$ an averaged operator if there is an $\alpha\in (0, 1)$ such that $A$ is $\alpha$-averaged.
\end{definition}
Note that the triangle inequality shows that an averaged operator is non-expansive. In addition, averaged operators allow for convergent fixed point iterations, whereas ordinarily non-expansive operators enjoy no such guarantees. This is of crucial importance in certain applications, such as Plug-and-Play algorithms, where modelling denoisers using non-expansive operators is not enough to prevent divergence, but using averaged operators can ensure convergence~\cite{sunOnlinePlugandPlayAlgorithm2019, hertrichConvolutionalProximalNeural2021}. For our analysis here, let us note the following fact:
\begin{lemma}
	\label{lemma:symmetric_jacobian_averaged}
	Suppose that $\Xi: \R^n\to\R^n$ is $C^1$ with symmetric Jacobian $D\Xi(x)$ everywhere and that $\alpha\in(0, 1)$. Then $\Xi$ is $\alpha$-averaged if and only if
	\[\spec(D\Xi(x)) \subset [1 - 2\alpha, 1]\]
	for all $x\in \R^n$. Note that the condition that the Jacobian is everywhere symmetric is equivalent to asking that $\Xi = \nabla f$ for some underlying functional $f:\R^n\to\R$.
\end{lemma}
\begin{proof}
	Showing that $\Xi$ is $\alpha$-averaged is equivalent to showing that $\Theta = (\Xi - \id)/\alpha + \id$ is non-expansive. Furthermore $\spec(\Theta) = \spec(\Xi)/\alpha + 1-1/\alpha $, so if $\spec(D\Xi(x))\subset [1-2\alpha, 1]$, then $\spec(D\Theta(x))\subset [-1, 1]$. By the fundamental theorem of calculus, we also have  that
	\[\Theta(x) - \Theta(y)  =\Big[\int\limits_0^1 D\Theta(y + t(x - y))\,\mathrm d t\Big] (x - y),\]
	so
	\[\|\Theta(x) - \Theta(y) \| \leq \Big[\sup_{z\in \R^n} \|D\Theta(z)\|\Big] \|x - y \|\leq \|x -y\|,\]
	as desired. Conversely, suppose that we do not have that $\spec(D(\Xi))\subset [1-2\alpha, 1]$ everywhere. In particular, there is some $x\in \R^n$, such that there is an eigenvector $v\in \R^n$ (assume that $\|v\|=1$) of $D\Theta(x)$ with eigenvalue outside of $[-1, 1]$: $D\Theta(x) v = \lambda v$ with $|\lambda| > 1$. Hence, we have
	\[\Theta(x + hv) - \Theta(x) = h D\Theta(x) v + R(h) = \lambda h v + R(h)\]
	with $R(h)$ a remainder term that satisfies $R(h) = o(h)$ as $h\to 0$. By the triangle inequality, this implies that
	\[\|\Theta(x + h v) - \Theta(x)\| \geq |\lambda| \|h v\| - |R(h)|.\]
	For small enough $h$, this inequality tells us that $\Theta$ is not non-expansive, and by the preceding argument, $\Xi$ is not $\alpha$-averaged.

\end{proof}
Recall that a single layer of the proposed architecture is given by $\Xi(x) = x - h A^\top\sigma(Ax + b)$, with the same setting in mind as described in~\Cref{lemma:arch_stiffness}. There we saw that $\Xi$ is the gradient of the functional $x\mapsto \|x\|^2/2 - h \phi(Ax + b)$, where $\phi$ is convex and $L$-smooth, so that $\spec(D^2\phi(x)) \subset [0, L]$ for each $x\in \R^n$. Hence, since we have $D\Xi(x) = \id - h A^\top D^2\phi(Ax + b) A$, we find that
\[\spec(D\Xi(x)) \subset [1 - h\|A\|^2 L, 1].\]
Combining this with~\Cref{lemma:symmetric_jacobian_averaged} immediately gives the following result if the activation function $\sigma$ is $C^1$. This is not required, however, for the result to be valid; any $L$-Lipschitz $\sigma$, such as $\sigma(x)=\texttt{ReLU}(x) = \max\{0, x\}$, will work equally well. The argument for general $L$-Lipschitz $\sigma$ is given below:
\begin{theorem}
	\label{thm:averaged_layer}
	Let $\sigma, A, b$ be as in~\Cref{lemma:arch_stiffness} and let $\alpha\in(0, 1)$. A single layer of the proposed architecture, $\Xi(x) = x - h A^\top\sigma(Ax+ b)$, is $\alpha$-averaged if
	\begin{equation}
		\label{ineq:averaged_layer_bound}
		h \|A\|^2 \leq 2\alpha /  L.
	\end{equation}
\end{theorem}
\begin{proof}
	Note that an operator $\Xi$ is $\alpha$-averaged if and only if $(\Xi - \id) / \alpha + \id$ is non-expansive. Furthermore, we note that $\Xi - \id = h f_{A, b}$, where $-f_{A, b}$ is the gradient of an $\|A\|^2L$-smooth convex functional, as defined in the proof of~\Cref{lemma:arch_stiffness}. By~\Cref{thm:bh} we have that
	\[  \frac{1}{\|A\|^2 L} \|f_{A, b}(x) - f_{A, b}(y)\|^2\leq\langle -f_{A, b} (x) +f_{A, b}(y), x -y\rangle\leq\|A\|^2 L \|x - y\|^2,\]
	so, if we write $\Lambda = (\Xi - \id)(x) - (\Xi - \id)(y)$ to reduce clutter, we have
	\[-h \|A\|^2 L \|x - y\|^2 \leq \langle \Lambda, x - y\rangle \leq  -\frac{1}{h\|A\|^2L} \|\Lambda\|^2.\]
	In particular, $\langle \Lambda, x -y\rangle + \|\Lambda\|^2 /(h\|A\|^2L)\leq 0$. Upon expanding the squared norm, we find that
	\begin{align*}
		\| ((\Xi - \id) / \alpha + \id)(x) - & ((\Xi - \id)/\alpha + \id)  (y)\|^2                                                                      \\&=\frac{\|\Lambda\|^2}{\alpha^2} + 2\frac{\langle \Lambda, x -y\rangle}{\alpha} + \|x - y\|^2 \\
		                                     & \leq \|x - y\|^2 +\frac{2}{\alpha}\Big(\langle \Lambda,x -y\rangle + \frac{\|\Lambda\|^2}{2\alpha}\Big).
	\end{align*}
	By the above comments, we see that $(\Xi - \id)/\alpha + \id$ is non-expansive when $2\alpha\geq h\|A\|^2L$, which can be rewritten into $h \|A\|^2 \leq 2\alpha / L$.
\end{proof}

Finally, the following result guarantees that the overall network will be averaged as long as each layer is averaged, with a corresponding $\alpha$ that can be controlled:
\begin{theorem}[Proposition 4.32 from~\cite{bauschke_convex_2011}]
	\label{thm:averaged_composition}
	\begin{sloppypar}  Suppose that $\Xi^1,\ldots,\Xi^m$ are operators $\Xi^i:\mathcal X\to \mathcal X$ on a Hilbert space $\mathcal X$ and that each $\Xi^i$ is $\alpha_i$-averaged for some $\alpha_i\in (0,1)$. Then $\Xi^m\circ \ldots\circ \Xi^1$ is $\alpha$-averaged, where
		\[\alpha = \frac{m}{m - 1+ \min\limits_{i=1,\ldots, m}(1/\alpha_i)}.\]
	\end{sloppypar}
\end{theorem}
In particular, if we are targeting a certain $\alpha\in (0, 1)$ for which our neural network ($m$ layers deep) should be $\alpha$-averaged, we should ask that each layer is $\alpha_i$-averaged with $\alpha_i$ at most
\[\alpha_i \leq \frac{\alpha}{m(1 - \alpha) + \alpha}.\]
By~\Cref{thm:averaged_layer}, we see that this implies that we must use a step size $h = \mathcal O(1/m)$ that decreases to $0$ as the depth $m$ of the network increases. Alternatively, it is possible to get an averaged operator by appealing to~\Cref{def:averaged}: $(1 - \alpha) \id + \alpha \Xi$ will be $\alpha$-averaged as long as $\Xi$ is non-expansive, which we have seen can be guaranteed with a step size independent of the depth of the network.
%%% Local Variables:
%%% mode: latex
%%% TeX-master: "main"
%%% End:

\section{Experiments}
\label{sec:experiments}
We will study the application of the proposed architectures to tasks where a certain kind of robustness is required: in~\Cref{sec:adv_robustness} we study the robustness of image classifiers to adversarial attacks and in~\Cref{sec:denoising} we consider the robustness of learned image denoisers. Robustness in the latter setting is of particular importance in downstream tasks, such as when the learned denoisers are used to solve inverse problems, as we will see in~\Cref{sec:pnp}. First, though, we will describe the methods that we use to train the networks.
\subsection{Training methods}
\label{sec:training_methods}
We train each network in a supervised manner, by solving a regularised empirical risk minimisation problem. In all experiments, we perform 40,000 iterations of stochastic gradient descent (SGD) with momentum (with momentum parameter $\beta = 0.9$), with a piecewise linear learning rate schedule that ramps up from a minimum learning rate to a maximum learning rate for the first half of the iterations and then down to 0 for the second half of the iterations. The minimum and maximum learning rates to use are found with the method described in~\cite{smithSuperConvergenceVeryFast2018}. Every experiment is run with weight decay, with weighting hyperparameters $\lambda \in \{10^{-5}, 5\cdot 10^{-5}, 10^{-4}, 5\cdot 10^{-4}\}$. The reported results correspond to the setting of the weight decay hyperparameter that performed best in terms of the loss function evaluated on a held-out validation set.

As mentioned before, to guarantee non-expansiveness or averagedness of the network we need to ensure that a bound such as~\Cref{ineq:averaged_layer_bound} in~\Cref{thm:averaged_layer} holds, so we use the power method~\cite{miyato_spectral_2018} to compute spectral norms of each of the learnable linear operators: if $A: \R^n \to\R^m$ is a linear operator and we have an initial estimate of the first left singular vector $v^0\in \R^m$, we iterate
\[u^k \leftarrow \frac{A^\top v^{k-1}}{\|A^\top v^{k-1}\|},\qquad v^k \leftarrow \frac{Au^k}{\|Au^k\|}.\]
Assuming that $v^0$ is not orthogonal to the first left singular vector (this is guaranteed to hold with probability $1$ if $v^0$ is randomly selected from a probability distribution that has a density w.r.t.\ Lebesgue measure), $u^k$ and $v^k$ converge to the first singular vectors of $A$ as $k\to\infty$ and $(u^k)^\top A v^k \to \|A\|$. In addition, it is reasonable to assume that the weights only undergo incremental updates during each training iteration, so that we can warm start the power method using the estimates of the singular vectors from the previous training iteration, followed by just a single iteration of the power method. In practice, the overhead of computing a single iteration of the power method for each block is significantly lower than the cost of forward and backward propagation through the network, as can be seen from the computational cost measurements in \ref{sec:computational_cost}.

The most straightforward way to satisfy the bounds that we require is by alternating gradient update steps with spectral normalisation steps $A\mapsto A/ \|A\|$. In a similar spirit, it is in principle also possible to differentiate through the spectral normalisation step~\cite{miyato_spectral_2018}, which results in different training dynamics and may avoid the tendency of the alternating steps to counteract each other's effect.

Instead of these methods, we find that another approach (which we will refer to as the adaptive approach in what follows) similarly ensures that the bounds of~\Cref{thm:averaged_layer} hold, with the benefits of empirically resulting in high expressive power, a simple implementation (compared to differentiating through the spectral normalisation) and a close connection to the underlying dynamical system. In the adaptive approach, we still use the power method to keep track of the operator norms, but when the bound in~\Cref{ineq:averaged_layer_bound} is violated, we split the integration interval into sufficiently many subintervals (each of equal size) to ensure that the bound is satisfied on each subinterval. Let us show an example of what this means for the approach shown in~\Cref{alg:euler_lipschitz}: after updating the weights of each block using SGD and updating the spectral norm estimate of the weights using the power method, the forward propagation shown in~\Cref{alg:euler_lipschitz} is modified. In a given block with weights $A^i$ and overall step size $h^i$, rather than take a single step $z^i \leftarrow z^{i - 1} - h^i (A^i)^\top\sigma(Az^{i - 1} + b)$, we compute a number of steps
\begin{equation}
    N=\lceil h^i \|A^i\|^2L / (2 r)\rceil
    \label{eq:adaptive_N_steps}
\end{equation}
to subdivide the integration into and take $N$ steps with stepsize $h^i / N$. In fact, since the ceiling defining $N$ is not generally attained (this can be checked after training to be certain), each block will naturally satisfy the constraint in~\Cref{thm:averaged_layer}, making the network averaged by~\Cref{thm:averaged_composition}. The explanation above, and especially \Cref{eq:adaptive_N_steps}, shows why we call this approach adaptive: during training, the operator norms of the linear operators in our non-expansive blocks are not fixed, but splitting each integration interval according to \Cref{eq:adaptive_N_steps} (which adapts to the operator norm $\|A^i\|$) ensures that the steps that we take are small enough to satisfy the bounds discussed in \Cref{sec:nonexpansive_odes} and \Cref{sec:euler_averaged}. Rather than adjusting the dynamics where our constraints are not satisfied as in the usual spectral normalisation approach, we allow the dynamics to evolve during training and adjust our approximation to ensure that the constraints are satisfied.

As a result of the adaptive approach, the depths of our networks are not fixed during training; although the number of learnable parameters is fixed, the network will become deeper during training if the weights grow. We do not take any additional measures to ensure that the weights stay bounded, but we find that this is not necessary in practice. A theoretical argument can be given for this too. Note that the loss functions that we consider are bounded from below. In fact, all of the loss functions that we consider are non-negative, i.e.\ bounded from below by 0. Besides this, we use weight decay, which is equivalent to adding a penalty function of the form $\theta\mapsto \lambda \|\theta\|^2$ to the loss function, where $\theta$ is a vector containing all parameters of the network in question and $\lambda >0$ is the weight decay hyperparameter. As a result, the overall objective function (including the regularisation from the weight decay) is coercive, meaning that it tends to $\infty$ as $\|\theta\|\to \infty$. From this we may deduce that any minimising sequence of the regularised loss function must be bounded. Since the operator norms $\|A^i\|$ of the linear operators in our non-expansive blocks are controlled by the norm of the overall parameter vector, $\|\theta\|$, this gives us that these operator norms also remain bounded along any minimising sequence. As a result, with the adaptive approach, the number of steps into which we subdivide any of the integration intervals also remains bounded, as can be seen from \Cref{eq:adaptive_N_steps}. Hence, we do not expect the depths of our networks to blow up during training.

To initialise the weights of the non-expansive networks that we study, we first use the default initialisation method in PyTorch~\cite{paszke_pytorch_2019} to initialise the convolutional filters and biases, and apply 1,000 iterations of the power method to compute the operator norms of the convolution operations. The filters are then normalised to have operator norm 1 and the singular vectors output by the power method are saved for future iterations.

In all of the networks that we study in the experiments, we choose the activation function $\sigma$ to be the $\texttt{LeakyReLU}$, defined as $\texttt{LeakyReLU}(x)  = \max\{ x, 0.01 x\}$. Evidently, $\texttt{LeakyReLU}$ is $1$-Lipschitz and corresponds to the gradient of a strongly convex functional. For each ODE block in the non-expansive networks, we fix the overall timestep to be the circle contractivity radius $r$ of the integrator used. Combining the weight initialisation mentioned above and the Lipschitz constant of $\texttt{LeakyReLU}$, this ensures that the required constraints are satisfied at initialisation. As previously described, the constraints are maintained during training using the adaptive approach.

All experiments have been implemented using PyTorch~\cite{paszke_pytorch_2019} and were run on a single computational node with a NVIDIA A100-SXM4 GPU with 80 GB of memory. The majority of training runs were finished in fewer than 4 hours (the exceptions being the image denoising experiment with higher order integrators in~\Cref{sec:denoising}, which took up to 12 hours to complete). The code that we have written to implement the methods and experiments is publicly available at \url{https://github.com/fsherry/non-expansive-odes}.

\subsection{Adversarial robustness}
\label{sec:adv_robustness}
Let us study the robustness of the proposed architecture as opposed to a comparable (but unconstrained) classifier based on a ResNet architecture. We will use the CIFAR10 dataset~\cite{krizhevskyLearningMultipleLayers2009}, which consists of 60,000 colour images of size $32\times 32$, with the standard split into 50,000 training images and 10,000 testing images. Additionally, we will hold out 2,500 training images to use as a validation set. Each image comes with one of 10 possible labels, which we will simply refer to as numbers in $\{1,\ldots, 10\}$. We normalise the images so that each channel only takes values in $[0, 1]$, but do not perform any additional normalisation or whitening. The neural networks that we consider implement functions between $\R^{32\times 32\times 3}$ and $\R^{10}$, the outputs of which are interpreted as the scores assigned to each label. Finally, given such a function $\Phi:\R^{32\times 32\times 3}\to \R^{10}$ its classification of an input image $x$ is given by $\argmax_{i\in \{1,\ldots, 10\}} \Phi(x)_i$.

Oftentimes, the scores output by a classifier are interpreted probabilistically: applying the softmax function, we transform the outputs of the network into probabilities, which can be fit to the true labels by minimising a cross-entropy loss function. In our setting, however, we will not take this approach, instead opting for a loss function that more directly encourages the classifier to have a large margin. Recall that the margin of a classifier at an input $x$ with true label $y$ is given by $\Phi(x)_y - \argmax_{i\in \{1,\ldots, 10\} \setminus \{y\}} \Phi(x)_i$. With this definition, images that are incorrectly classified correspond to a negative margin, while those that are correctly classified correspond to a positive margin. If we can lower bound the margin at a correctly classfied image, while upper bounding the Lipschitz constant of $\Phi$ we can certify that $\Phi$ equally classifies all images in a ball around this image~\cite{tsuzukuLipschitzMarginTrainingScalable2018}. In particular, maximising the margin of a classifier makes sense when we constrain its Lipschitz constant. In this context, it is natural to use a multi-class hinge loss function:
\begin{equation}
	\min_{\Psi} \frac{1}{N_\textrm{train}} \sum_{i=1}^{N_\textrm{train}} \sum_{k=1}^{10} \max\{0, \mu - \Phi(x_i)_{y_i} - \Phi(x_i)_k\},
	\label{prob:loss_classification}
\end{equation}
with $\mu$ a hyperparameter that we fix to be equal to $0.1$ in our experiments. This choice was made by considering values of $\mu \in \{0.1, 0.3, 0.5, 0.7, 1\}$ and comparing the validation accuracies of the baseline ResNet (described below) trained with each of these values. It is worth noting that we found the validation accuracies for all of these choices of $\mu$ to be very close to each other, suggesting that the results are relatively insensitive to the choice of $\mu$. Furthermore, we will use data augmentation, effectively replacing the inner term in~\Cref{prob:loss_classification} by an expectation over the augmentations, and $x_i$ by $\Delta(x_i)$ with $\Delta$ the augmentation. We compose an augmentation that randomly erases pixels, an augmentation that randomly crops out parts of the images and an augmentation that applies a random horizontal flip. For the training of all networks we use minibatches of size 128 in each SGD update.

We will consider two comparable network architectures, which are essentially based on the wide ResNet architecture~\cite{zagoruykoWideResidualNetworks2016}, which has been shown to perform very well for image classification: an architecture we will refer to as NonExpNet which uses our proposed nonexpansive components, and an architecture we will call ResNet which is similar but unconstrained. In particular, by appropriately setting the weights of the ResNet, we will be able to recover a NonExpNet. Both classifiers take the form
\[\Phi = A_\textrm{lin} \circ \textrm{pool}_{\textrm{global}} \circ \mathop{\bigcirc}_{i=1}^3 \Big[\Xi^i \circ \textrm{pool}_{2\times 2}\circ A_{\textrm{lift},i} \Big] .\]
Here $A_\textrm{lin}$ is a final linear layer, $A_{\textrm{lift}, i}$ are $1\times 1$ convolutions increasing the number of channels, $\textrm{pool}_{2\times 2}$ are average pooling operations with $2\times 2$ kernels and a stride of $2$ and $\textrm{pool}_{\textrm{global}}$ is a global average pooling operation (which collapses the spatial extent of the remaining image into a single pixel). In the case of the ResNet, each $\Xi^i$ is a composition of 5 simple ResNet blocks of the form $x \mapsto B\sigma(Ax +b)$, whereas in the case of the NonExpNet, each $\Xi^i$ takes the form shown in~\Cref{alg:euler_lipschitz} with 5 blocks. Setting $B=-A^\top$ in the ResNet, we recover the form of the NonExpNet, showing that the ResNet is at least as expressive as the NonExpNet before it has been allowed to grow deeper with the adaptive training method. Since the adaptive training method deepens the NonExpNet, but does not add parameters, we do not expect the expressiveness of the NonExpNet to change significantly as it grows deeper. With the particular configuration used in our experiments, the NonExpNet has 254,810 trainable parameters, while the ResNet has 497,290 trainable parameters. Recall from~\Cref{sec:training_methods} that we use an adaptive training approach for NonExpNet, meaning that its depth is generally allowed to grow during training, but its number of trainable parameters remains constant throughout training. This is done to enforce the constraints that we require to ensure nonexpansiveness. On the other hand, since we do not enforce constraints on the ResNet, its depth is constant during training. We report results for the weight decay hyperparameters that we found to give the highest accuracy on the validation set: for the baseline ResNet, this corresponds to $\lambda = 10^{-5}$, and for NonExpNet we found that $\lambda =10^{-4}$ was the best choice. For the experiment in which we considered adversarial training to train the baseline ResNet, we reused the weight decay hyperparameter that we found to be best for ``normally'' training the baseline ResNet.

To implement the adversarial attacks, we use the Foolbox package~\cite{rauber2017foolboxnative}, applying a projected gradient descent (PGD) attack (which in fact alternates gradient \emph{ascent} steps on the loss function with projections onto a ball around the clean image). More specifically, we consider an attack that uses an $\ell^2$-norm constraint and perform $\texttt{n\_iter}=100$ iterations of the projected gradient descent method (see~\Cref{alg:l2_pgd}), as in~\cite{madryTowardsDeepLearning2018}. Differently than in~\cite{madryTowardsDeepLearning2018}, where the main focus is on $\ell^\infty$-norm constraints, we consider the $\ell^2$ threat model as the proposed networks have stability guarantees with respect to the $\ell^2$-norm. Projected gradient descent is a fully white-box attack (it assumes that we can evaluate and backpropagate through the classifier), and depending on the loss function $\ell$ used, we can generate different kinds of attacks, including untargeted or targeted attacks. While we use the multi-class hinge loss of \Cref{prob:loss_classification} for training, we use the cross-entropy loss,
\[\ell(\hat y, y) = -\sum\limits_{k=1}^{10} y_k \log(\hat y_k),\]
within \Cref{alg:l2_pgd} to generate untargeted attacks to evaluate robustness, as is standard procedure in the literature~\cite{madryTowardsDeepLearning2018}.

\begin{algorithm}[ht!]
	\caption{$\ell^2$-PGD adversarial attack}
	\label{alg:l2_pgd}
	\begin{algorithmic}
		\State \textbf{inputs:} image $x$, class $y$, loss function $\ell$, classifier $\Phi$, step size $\tau$, perturbation size $\eps$, number of iterations $\texttt{n\_iter}$
		\State Sample initial random $\delta\in B_{\ell^2}(0, \eps)$
		\For{$i\leftarrow 1,\ldots, \texttt{n\_iter}$}
		\State $g\leftarrow \nabla_\delta \ell(\Phi(x + \delta), y)$
		\State $\delta \leftarrow \delta + \tau g / \|g \|$
		\State $\delta \leftarrow \textrm{project}_{B_{\ell^2}(0, \eps)}(\delta)$
		\EndFor\\
		\Return $\delta$
	\end{algorithmic}
\end{algorithm}

Besides comparing to the baseline ResNet, trained in the usual way, by solving the empirical risk minimisation problem of~\cref{prob:loss_classification}, we will compare to an adversarially trained ResNet, which uses exactly the same architecture, but is explicitly trained for adversarial robustness. We will refer to this adversarially trained ResNet as ResNet-AT. The training procedure is similar to the usual training by SGD, except that at each iteration, the minibatch of clean images $\{x_i\}_{i}$ is replaced by a minibatch of corresponding perturbed images $\{\tilde x_i=x_i + \delta_i \}_i$ generated using \cref{alg:l2_pgd}. Since the adversarial perturbations depend on the state of the network, these adversarial minibatches must be regenerated at each iteration, which takes an additional $\texttt{n\_iter}$ forward and backward passes through the network, so that total training time is multiplied by a factor of $\texttt{n\_iter} + 1$ compared to the ordinary ResNet. As a result, it is necessary to choose \texttt{n\_iter} smaller during adversarial training than at test time to ensure computational feasibility. We choose $\texttt{n\_iter}=7$ for adversarial training and generate adversarial samples according to \cref{alg:l2_pgd} with $\eps=0.5$.

\begin{figure}[!htb]
	\centering
	\includegraphics[scale=1]{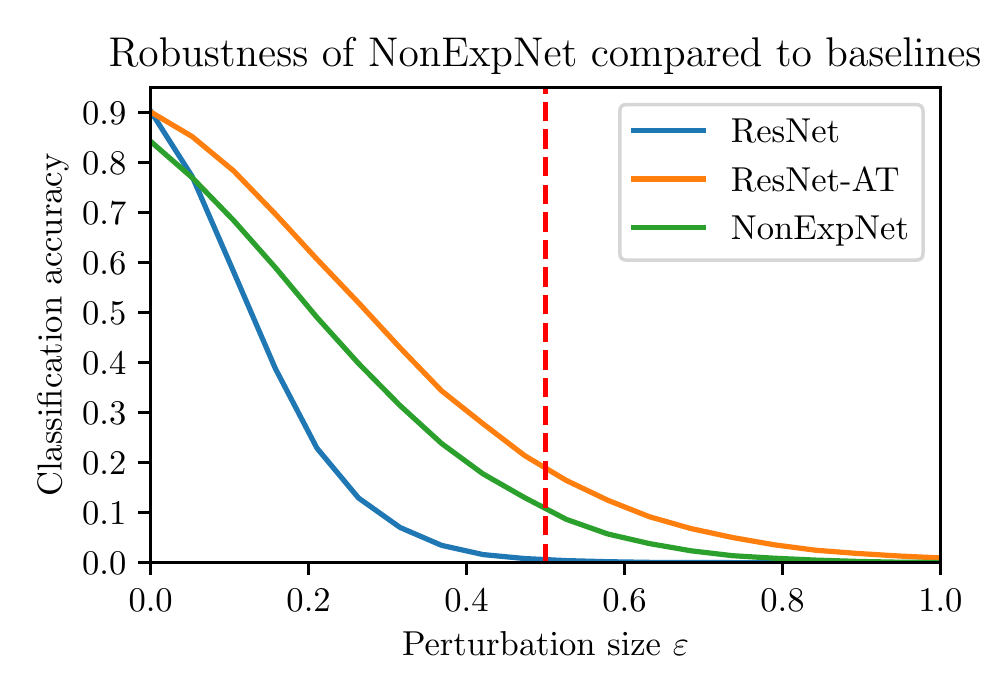}
	\caption{A comparison of the adversarial robustness of the proposed NonExpNet and a similar ResNet trained to classify CIFAR10 images. The baseline ResNet was trained both ordinarily (simply denoted ``ResNet'') and with adversarial training (denoted ``ResNet-AT''). The dashed red line indicates the perturbation size that was used to generate the adversarial examples for adversarial training. The adversarial attack takes the form of an $\ell^2$-PGD attack with 100 iterations, with the robust accuracy curves computed on the test set.}
	\label{fig:robustness_curves}
\end{figure}

The result of applying such an attack to the networks described above is shown in~\Cref{fig:robustness_curves}: we compute attacks of various sizes $\eps$ (corresponding to the size of the $\ell^2$-ball within which we search for an attack) for each image in the test set and compute how well both networks continue to perform on the attacked images. Clearly, the unconstrained ResNet has the highest clean accuracy (accuracy at $\eps=0$), namely 90.2\%, in comparison to the adversarially trained ResNet-AT, which achieves a clean accuracy of 90.1\% and the NonExpNet, which achieves a clean accuracy of 84.2\%. Because we are not partial to any particular perturbation size, we compute the areas under the curves (AUCs) in \Cref{fig:robustness_curves} to get a quantitative impression of the robustness. In particular, this can be interpreted as the average robust accuracy over the range of perturbation sizes that we are studying. This metric shows ResNet-AT to be most robust, with an AUC of 0.304 (average robust accuracy of 30.4\%), NonExpNet next with an AUC of 0.234 (average robust accuracy of 23.4\%), and the ResNet least robust, with an AUC of 0.142 (average robust accuracy of 14.2\%). Since we have noted that the ResNet can generally be expected to be at least as expressive as the NonExpNet, the difference in clean accuracy may be explained by the difference in expressiveness. More striking is the behaviour that can be qualitatively observed in \Cref{fig:robustness_curves}: the slopes of the plots for NonExpNet and ResNet-AT are almost equal and significantly smaller than the slope of the plot for the baseline ResNet. In fact, the difference in AUC between NonExpNet and ResNet-AT (7\%) can almost completely be explained by the shift corresponding to the difference in their clean accuracies (6.9\%). Bearing in mind that NonExpNet was trained with a ``normal'' training procedure (i.e.\ no adversarial images were seen during training), this shows the effectiveness of the inductive bias of non-expansiveness by itself in improving robustness of ResNet-styled image classifiers.

\subsection{Nonexpansive neural networks for denoising}
\label{sec:denoising}
Let us now proceed by studying the performance of the proposed networks on an image denoising task, comparing to a standard variational denoising approach and an unconstrained deep learning approach. We will use the BSDS500 dataset~\cite{arbelaez_contour_2011} for these experiments. This dataset consists of 500 natural colour images of size $321\times 481$, split into $N_\text{train}=200$ training images, $N_\text{val}=100$ validation images and $N_\text{test}=200$ test images. In our experiments, we adhere to the same splitting of the dataset. We normalise the images so that each channel only contains values in $[0, 1]$ and simulate noisy images $y$ corresponding to each ground truth image $x$ by adding Gaussian white noise $\eps$ with a standard deviation of $0.15$, corresponding to a high noise level (see~\Cref{fig:denoisers_kayak,fig:denoisers_bird}). For all of the learned approaches, we train the networks $\Gamma$ involved to solve the empirical risk minimisation problem
\begin{equation}
	\min_\Gamma \frac{1}{N_\text{train}}\sum\limits_{i=1}^N \| \Gamma(y_i) - x_i \|^2,
	\label{prob:learning_denoiser}
\end{equation}
and when we test the performance of the various denoisers on the test set, we essentially use the same loss: we report the peak signal-to-noise ratio (PSNR), defined between a reference $x^*$ and an estimate $\hat x$ as
\[\text{PSNR}(\hat x, x^*) = 10 \log_{10}\Big(\frac{\max_{i, j, k} |x^*_{i, j, k}|^2}{\frac{1}{3\cdot 321\cdot 481}\sum_{i, j, k} |x^*_{i, j, k} - \hat x_{i, j, k}|^2}\Big) .\]
The PSNR is a common metric used to measure image quality. It is a monotonically decreasing transformation of the mean squared error, normalised by the values of the reference image (so that it is invariant under the scaling $(\hat x, x^*)\mapsto (\lambda\cdot \hat x, \lambda\cdot x^*)$ for any $\lambda\in\R\setminus \{0\}$). As a result, higher PSNR values can be taken as an indication of better reconstructions. We use the training approach described in~\Cref{sec:training_methods} and consistently use minibatches of size 5.

The networks that we propose for this task are of the form
\begin{equation}
	\Gamma = A_\text{project}\circ\Xi \circ A_\text{lift},
	\label{eq:denoiser_definition}
\end{equation} where $A_\text{lift}$ is a lifting operation taking the $3$ input channels to $64$ channels by appending 61 channels filled with zeros, $A_\text{project}$ is a projection operator taking $64$ channels to the $3$ output channels by simply dropping the last 61 channels, and $\Xi$ is a network as in~\Cref{alg:euler_lipschitz} with each $A^i$ a convolution taking $64$ channels to $64$ channels and each $b^i\in \R^{64}$ (i.e.\ the biases are spatially constant, as usual). All convolution operators have kernel size $3\times 3$. This choice of lifting and projection operators ensures that $\alpha$-averagedness of $\Xi$ implies $\alpha$-averagedness of $\Gamma$: we have $A_\text{project} \circ A_\text{lift}=\id$ and $\|A_\text{lift}\| = \|A_\text{project}\|=1$ so if $\Xi = (1-\alpha) \id + \alpha T$ is $\alpha$-averaged with $T$ non-expansive, we also have
\begin{align*}
	\Gamma & =A_\text{project} \circ( (1- \alpha) \id + \alpha T ) \circ A_\text{lift}                           \\
	       & =(1-\alpha) A_\text{project}\circ A_\text{lift} + \alpha A_\text{project}\circ T\circ A_\text{lift} \\
	       & = (1-\alpha)\id + \alpha A_\text{project}\circ T\circ A_\text{lift}.
\end{align*}
Since $A_\text{project}$ and $A_\text{lift}$ are non-expansive, this shows that $\Gamma$ is $\alpha$-averaged. We will instantiate this architecture using the forward Euler integrator as in~\Cref{alg:euler_lipschitz} (with $\texttt{n\_blocks} = 10$) and will simply refer to this network by the name Euler in what follows.

Although the concrete architecture obtained when using the forward Euler integrator (as described in~\Cref{alg:euler_lipschitz}) is appealing in its simplicity, the framework laid out in~\Cref{sec:nonexpansive_odes} also allows us to use certain higher order integrators. For instance, consider Heun's method, which is given by \begin{equation}
	\label{eq:heun_method}
	\Phi^\text{Heun}_h(t, y, f) = y + \frac{h}{2} \Big(f(t, y) + f(t + h, y + hf(t, y))\Big).
\end{equation}
This is a $2$-stage, second-order, RK method, with
\[\mathcal A = \begin{pmatrix}0&0\\1&0\end{pmatrix}, \quad b = \begin{pmatrix}1/2\\1/2\end{pmatrix},\quad \diag(b)\mathcal A + \mathcal A^\top \diag(b) - bb^\top =\frac{1}{4}\begin{pmatrix} 1&-1\\-1&1\end{pmatrix},\]
and using~\Cref{remark:computing_r}, we conclude that Heun's method is $1$-circle contractive, like the forward Euler method is. As a result, $x\mapsto\Phi^\text{Heun}_h(0, x, f_{A, b})$ is non-expansive as long as $h \|A\|^2\leq 2$ and~\Cref{alg:euler_lipschitz} can be adapted to use Heun's method, the only change being that the steps $z^i\leftarrow z^{i -1} - h (A^i)^\top \sigma(A^i z^{i - 1} + b^i)$ are replaced by steps of the form $z^i \leftarrow \Phi^{\text{Heun}}_h(0, z^{i - 1}, f_{A^i, b^i})$. We will also instantiate the architecture in~\Cref{eq:denoiser_definition} with this integrator (again using $\texttt{n\_blocks} = 10$) and will refer to this network by the name Heun. As this integrator takes more evaluations of the right hand side of the ODE, it is more costly to compute the output of Heun than of Euler. Interested readers are referred to \ref{sec:timing_denoising} for comparisons of the time and memory requirements of these methods in the setting that we consider.

Similarly, we can consider integrators with yet higher orders, such as the fourth-order RK4 integrator $\Phi^{\mathrm{RK4}}_h$, given by~\Cref{def:runge_kutta} with
\begin{align*}
	 & \mathcal A = \begin{pmatrix} 0                                                            & 0 & 0  & 0  \\ 1/2 & 0 & 0 & 0 \\ 0 & 1/2 & 0 & 0\\ 0 & 0 & 1 & 0\end{pmatrix},\qquad\qquad\qquad\qquad b = \begin{pmatrix}1/6\\1/ 3\\1/3\\1/6\end{pmatrix}, \\
	 & \diag(b) \mathcal A + \mathcal A^\top\diag(b) - b b^\top = \frac{1}{36}\begin{pmatrix} -1 & 4 & -2 & -1 \\ 4 & -4 & 2 & -2\\-2 & 2 & -4& 4\\-1 & -2 & 4 & - 1\end{pmatrix}.
\end{align*}
Again using~\Cref{remark:computing_r}, we conclude that the RK4 method is $1$-circle contractive and we can replace the forward Euler method in~\Cref{alg:euler_lipschitz} by the RK4 method to obtain a non-expansive neural network. We design a denoiser (with $\texttt{n\_blocks}=10$) as in~\Cref{eq:denoiser_definition} with the RK4 integrator and refer to it by the name RK4. RK4 is again more costly to use since the RK4 integrator takes 4 evaluations of the right hand side of the ODE for each step.

All of our proposed networks have the same number, 369,280, of trainable parameters.

As a benchmark denoising algorithm using the variational approach, we can consider total variation (TV) denoising~\cite{rudin1992}, which gives the denoised image as
\begin{equation}
	\hat u = \argmin_{u} \frac{1}{2}\|u - y\|^2 + \alpha \| \nabla u \|_1,
	\label{eq:tv_denoising}
\end{equation}
where we have tuned $\alpha$ for optimal reconstruction performance on the training set. This approach solves a convex optimisation problem with a hand-crafted regularisation functional that favours reconstruction of piecewise constant images. Solving this problem accurately is more expensive than the deep learning approaches that we will consider and generally its hand-crafted prior information is not matched as well to reality as the prior information encoded in the learned approaches. On the other hand, this approach can be considered trustworthy for use in downstream tasks: \Cref{eq:tv_denoising} defines the reconstruction map as a proximal operator of a convex functional, so that the reconstruction map is $1/2$-averaged.

The DnCNN, introduced in~\cite{Zhang2017}, has become a standard benchmark for denoising tasks. A natural comparison to make is between our networks $\Gamma$ with $\texttt{n\_blocks} = 10$, and the DnCNN $\Gamma_\text{DnCNN} = \id - A_\text{project}\circ \Xi_\text{DnCNN} \circ A_\text{lift}$ where $\Xi_\text{DnCNN}$ is a 18-layer convolutional neural network without skip connections, the details of which are described in~\cite{Zhang2017}. $A_\text{lift}$ and $A_\text{project}$ are taken as convolutions that lift the channels to 64 channels from the 3 input channels and project the 64 channels back to the 3 output channels respectively. Note that this is a reasonable version of the DnCNN to compare with our networks since each block of our architecture contains a convolution and its transpose, whereas the DnCNN uses one convolution per layer. This DnCNN is trained to solve~\Cref{prob:learning_denoiser}, in the same way as our architectures except that no operator norm constraints are enforced on the convolutions. The DnCNN that we consider has 670,531 trainable parameters. The weight decay hyperparameters that we found to give the highest mean PSNR on the validation set were $\lambda =5\cdot 10^{-4}$ for the DnCNN and $\lambda = 10^{-5}$ for the non-expansive network using the Euler integrator. For the higher order integrators, we reused the weight decay hyperparameter that was optimal for the Euler integrator.

\begin{figure}[!htb]
	\centering
	\includegraphics[scale=1]{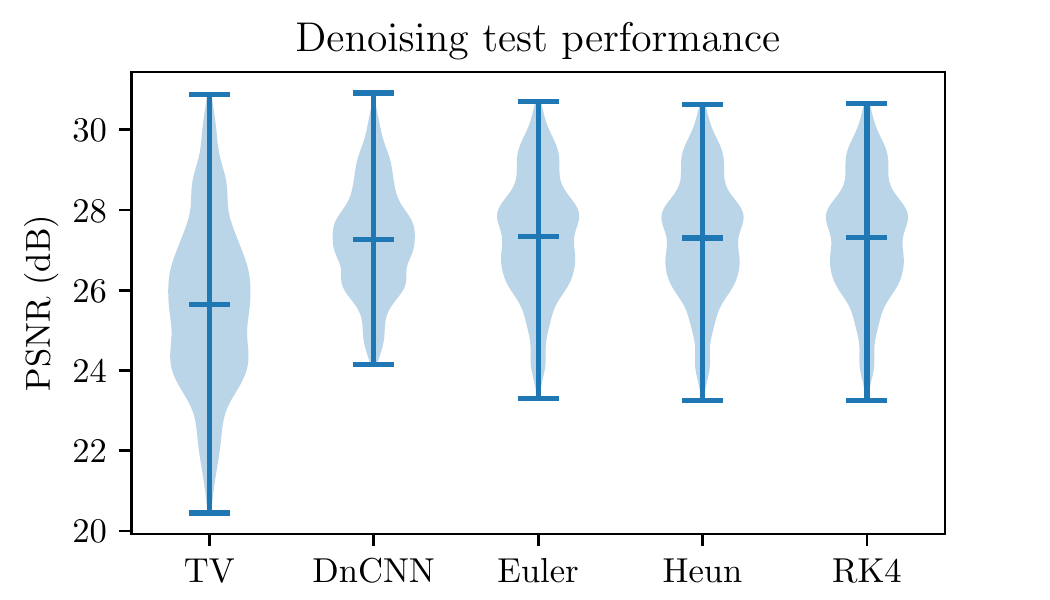}
	\caption{Comparison of the denoising performance of the various considered approaches on the test set. For each image in the test set, 10 instances of noisy images were generated, so that each denoiser is tested on a total of 2000 noisy images. The learned approaches all perform similarly and outperform the variational approach that uses TV regularisation.}
	\label{fig:denoisers_violinplot}
\end{figure}

The results of comparing all of these approaches are shown in~\Cref{fig:denoisers_violinplot}: all of the learned approaches perform basically on par with each other, slightly outperforming the variational approach. The gap between the learned approaches and the variational approach is not as large as one may expect, but this is a result of the high noise level considered. \Cref{fig:denoisers_kayak} and~\Cref{fig:denoisers_bird} show zoomed in examples of test images that are relatively more favourable to the learned approaches and to TV and vice versa. It is remarkable that Euler, Heun and RK4 perform extremely similarly to each other: although the higher order integrators are more costly to evaluate, and attain higher order of approximation of the underlying trajectories, this is not borne out in better denoising performance. This can be explained by the fact that (at least in this task) we are only interested in the endpoints of the underlying trajectories.

Another point that is worth expanding on is the fact that our proposed networks perform on par with the unconstrained DnCNN. This is perhaps somewhat unexpected: recall from the discussion in~\Cref{sec:nonexpansive_odes} that our approach bears some similarity to the Parseval proximal networks described in~\cite{hasannasab_parseval_2020,hertrichConvolutionalProximalNeural2021}. Their approach effectively uses an implicit Euler integrator and restricts the weights to be orthogonal to make the computation of the implicit step computationally feasible, and they observed that it was necessary to scale up their proposed networks by a multiplicative factor and apply residual learning (where the network models the noise instead of the image directly) to obtain networks that perform similarly to DnCNN. These modifications break the desirable non-expansiveness properties of the network, necessitating tricks such as using a so-called ``oracle denoiser'' to obtain the desired properties. On the other hand, our experiments in this section show that these tricks are not necessary: it is possible to enforce stability constraints while having high denoising performance.

\begin{figure}[!htb]
	\centering
	\includegraphics[scale=1]{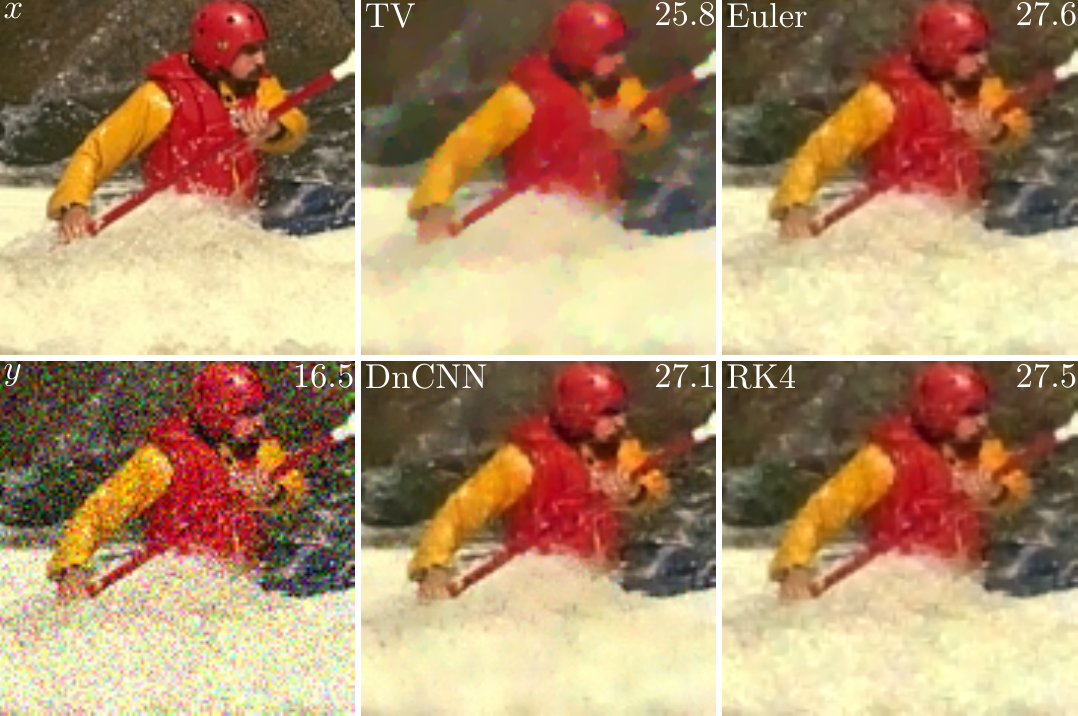}
	\caption{Comparison of some of the denoisers considered, on an image that is favourable to the learned approaches. Note that the learned approaches recover more fine details, whereas TV has a tendency to flatten them out. The numbers in the top right corner of each image are the PSNRs (in dB) relative to the ground truth $x$.}
	\label{fig:denoisers_kayak}
\end{figure}

\begin{figure}[!htb]
	\centering
	\includegraphics[scale=1]{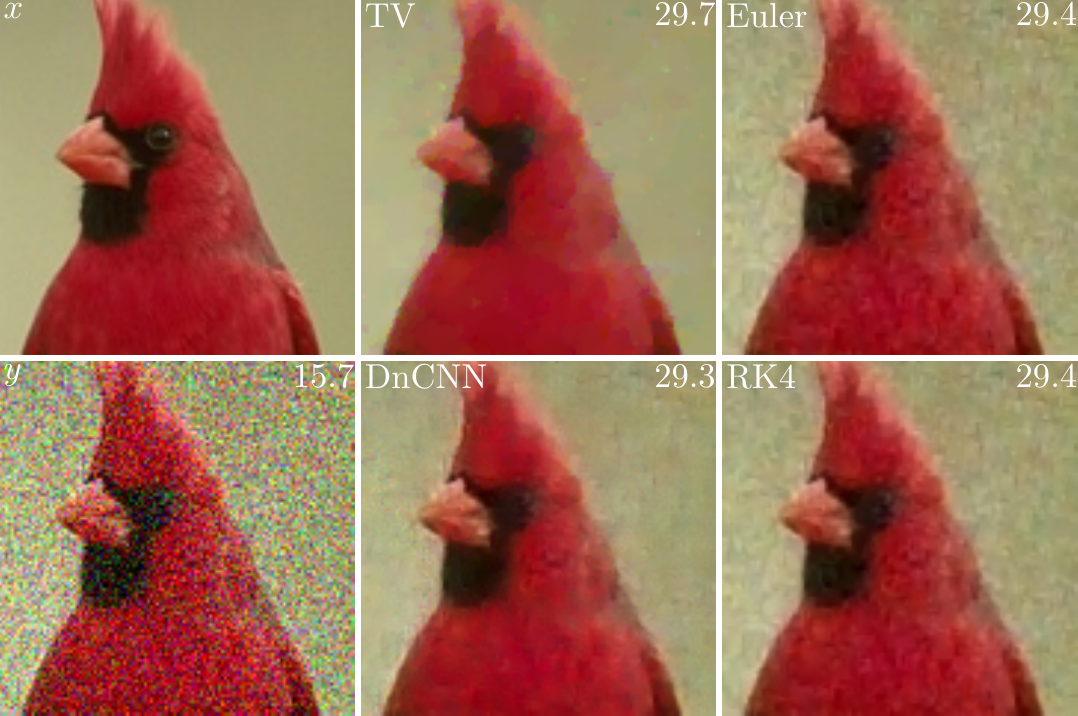}
	\caption{Comparison of some of the denoisers considered, on an image that is favourable to TV. The true image to be recovered is relatively well-approximated by a piecewise constant image. Note that the finer details in the image (for example the ridge between the top and bottom part of the beak) are still better recovered by the learned approaches. The numbers in the top right corner of each image are the PSNRs (in dB) relative to the ground truth $x$.}
	\label{fig:denoisers_bird}
\end{figure}
\subsection{Plug-and-Play applications of the learned denoisers}
\label{sec:pnp}
Recall from~\cref{sec:introduction} that one of the downstream applications where provably stable neural networks are of great interest is in the Plug-and-Play (PnP) approach to solving ill-posed inverse problems using learned denoisers. With the results from the previous section and the theoretical results of~\cref{sec:euler_averaged}, we are now in a position to pursue this further. Indeed, recalling~\cref{thm:averaged_layer} and~\cref{thm:averaged_composition}, we get that the denoiser trained using the Euler integrator in the previous section is an averaged operator, since the adaptive training approach ensures that the condition in~\cref{thm:averaged_layer} is satisfied (for an unspecified $\alpha > 0$) at all times. In particular, it has convergent fixed-point iterations~\cite[Theorem 5.14]{bauschke_convex_2011}. Let us see what this means in practice: \cref{fig:repeated_denoisers} shows what happens if we repeatedly apply the learned DnCNN and Euler denoisers to an input image. We clip the values of all images to lie in $[0, 1]$ for display purposes, but it is evident from the figure that repeatedly applying DnCNN results in divergence, whereas repeatedly applying Euler converges. This convergence property is a desirable property for a denoiser, and is for instance satisfied by any proximal operator of a convex regularisation functional.

\begin{figure}[!htb]
	\centering
	\includegraphics[scale=1]{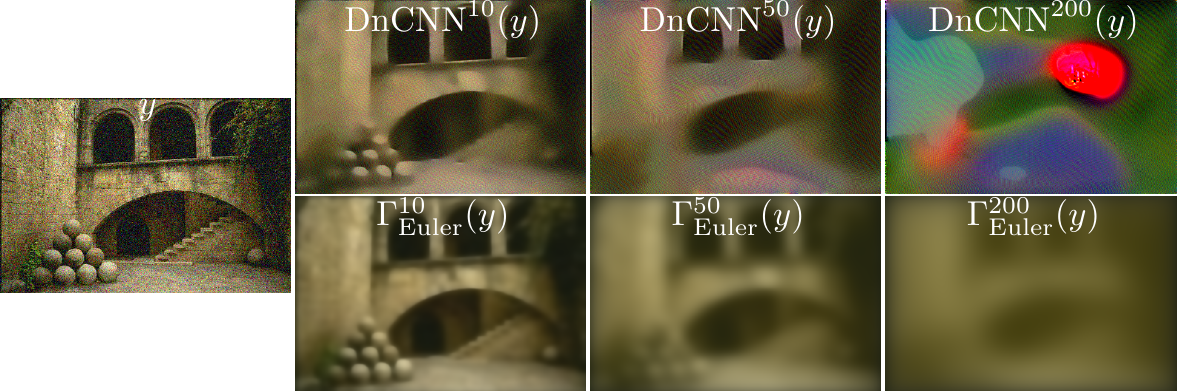}
	\caption{Repeated application of the unconstrained and constrained denoisers to a given input image gives drastically different results: for the unconstrained DnCNN, this sequence diverges, whereas for the averaged Euler denoiser, this sequence converges.}
	\label{fig:repeated_denoisers}
\end{figure}

Let us consider the inverse problem of deblurring: we assume that we are given measurements $y = K x + \eps$, where $K x = k* x$ is a convolution operation representing a motion blur, $x$ is the ground truth image that we aim to recover as well as possible and $\eps$ is Gaussian white noise corrupting the measurements. The ill-posedness of this problem is manifested in the instability of the inverse of the convolution; as a consequence of this, a naive inversion of the measurements will blow up the noise in the measurements. The PnP approach uses the previously learned denoiser to regularise this inversion. In particular, we will use the simple PnP proximal gradient method of~\cref{alg:pnp_fb}, with $E_y(x) = \| y - Kx \|^2/2$ and the denoiser $\Phi$ equal to the previously learned Euler denoiser. We choose the step size $\tau = 1 / \| K\|^2$, which together with the fact that the denoiser is averaged gives a convergence guarantee~\cite[Proposition 2]{sunOnlinePlugandPlayAlgorithm2019}. The asserted convergence is practically observed in~\cref{fig:pnp_figures}: the result $\hat x$ obtained as the limit of the PnP iterations trades off the prior information about natural images encoded in the Euler denoiser with the data consistency.

\begin{figure}[!htb]
	\centering
	\includegraphics[scale=1]{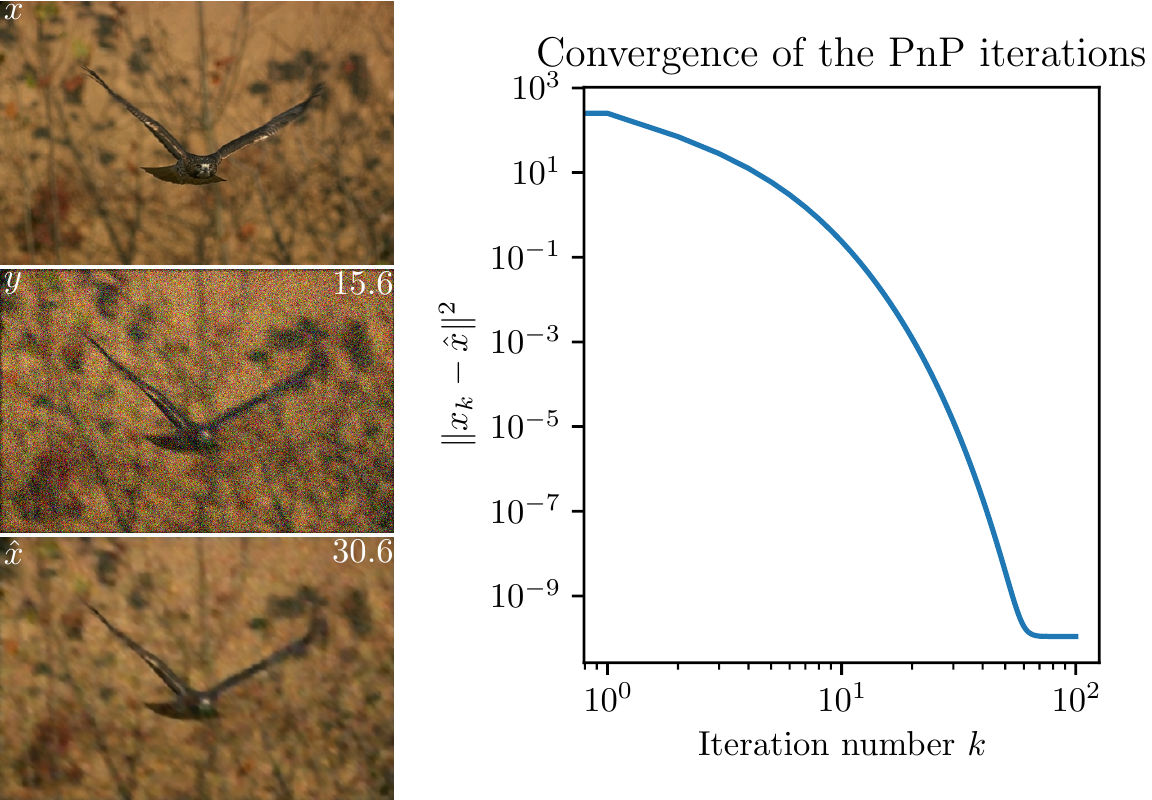}
	\caption{Using the learned Euler denoiser to solve an ill-posed inverse problem (deblurring) in a PnP fashion, with convergence guarantee. The numbers in the top right corner of each image are the PSNRs (in dB) relative to the ground truth $x$.}
	\label{fig:pnp_figures}
\end{figure}

%%% Local Variables:
%%% mode: latex
%%% TeX-master: "main"
%%% End:

\section{Conclusions and discussion}
We have exhibited a family of ResNet architectures for which it is straightforward to enforce non-expansiveness. The proposed architecture is given by compositions of numerical integration steps along gradient flows in convex potentials. For the main example using the forward Euler method, we have used tools from convex analysis to show that the architecture can be used to encode averaged operators. We have demonstrated the use of the proposed architectures on adversarially robust image classification and on image denoising, with the idea of applying the learned denoisers to ill-posed inverse problems.

With a novel adaptive training approach, we have shown that it is possible to obtain performant neural networks, even when enforcing desirable stability constraints. Although the basic architecture uses the first order forward Euler method as the numerical integrator, it is possible to use higher order methods. In the tasks considered in this work, using higher order integrators did not come with any benefits: the performance of the learned networks was not improved, while the computational cost was significantly increased. It remains a question of interest whether there are tasks where the higher order integrators provide real benefits over the forward Euler integrator.

Future work could also go towards studying the application of these architectures in typical deep learning applications such as GANs, specifically in Wasserstein GANs which require the use of a $1$-Lipschitz critic function. We have seen in practice that the proposed architectures are quite expressive, but an interesting direction for future work would also be to study ways in which more general learnable non-expansive flows can be used to motivate the design of provably stable neural network architectures, and provide approximation guarantees for them. This is of particular interest in the setting of adversarial robustness. While we did not see a meaningful difference in performance between our proposed networks and the unconstrained, high-performance, denoiser in \cref{sec:denoising}, we saw a gap in clean accuracy between the proposed NonExpNet and an unconstrained ResNet in the image classification task in \cref{sec:adv_robustness}. More to the point, we saw that the unconstrained ResNet trained using adversarial training can attain similar robustness to the NonExpNet, while still being highly accurate on clean images. In the future, we aim to improve NonExpNet to narrow the gap in clean accuracy, with the aim of retaining robustness without requiring adversarial training. Finally, the framework shown in this paper essentially depends on the use of the Euclidean norm, but depending on the application one may want to design neural networks based on the dynamical systems connection that are 1-Lipschitz with respect to different norms.

%%% Local Variables:
%%% mode: latex
%%% TeX-master: "main"
%%% End:

\section*{Acknowledgments}
EC \& BO have received support from the European Union Horizon 2020 research
and innovation programme under the Marie Sk\l odowska-Curie grant agreement No.\ 860124. MJE acknowledges support from EPSRC (EP/\allowbreak S026045/\allowbreak 1, EP/T026693/1, EP/V026259/1) and the Leverhulme Trust (ECF-2019-478). DM was partially supported by a grant from the Simons Foundation. CBS acknowledges support from the Philip Leverhulme Prize, the Royal Society Wolfson Fellowship, the EPSRC advanced career fellowship EP/V029428/1, EPSRC grants EP/S026045/1 and EP/T003553/1, EP/N014588/1, EP/T017961/1, the Wellcome Innovator Awards 215733/Z/19/Z and 221633/Z/20/Z, the European Union Horizon 2020 research and innovation programme under the Marie Sk\l odowska-Curie grant agreement No.\ 777826 NoMADS, the Cantab Capital Institute for the Mathematics of Information and the Alan Turing Institute. FS acknowledges support from the EPSRC advanced career fellowship EP/V029428/1.

\bibliographystyle{plain}
\bibliography{references}

\newpage
\appendix

\section{Comparisons of computational costs}
\label{sec:computational_cost}
In this section, we will give estimates of relevant timings and memory measurements for the methods that we compared in~\Cref{sec:experiments}. In each instance, we perform these measurements on a single minibatch, with the same settings as were used for training, so that the timings give a reasonable view of the cost of training.

To perform the timing and memory measurements, we use functionality provided by PyTorch for NVIDIA GPUs. Replacing the comments in Code Snippet~\ref{snippet:GPU_timing_memory} by appropriate code to set up the model and the code to be measured, we get code that can be used for these measurements.

\begin{code}
    \captionof{listing}{GPU timing and memory measurement code}
    \begin{minted}[frame=lines]{python}
import torch

""" code to set up model on device """

start = torch.cuda.Event(enable_timing=True)
end = torch.cuda.Event(enable_timing=True)
torch.cuda.reset_max_memory_allocated(device)

start.record()
""" run code involving model """
end.record()
torch.cuda.synchronize(device)

time = start.elapsed_time(end)
memory = torch.cuda.max_memory_allocated(device)
    \end{minted}
    \label{snippet:GPU_timing_memory}
\end{code}

In the tables that follow, $t_\text{FB}$ is the time taken to perform a forward propagation and backpropagation, $t_\text{P}$ is the time taken to perform an iteration of the power method, $t_\text{F}$ is the time taken to perform a forward propagation and we refer to the GPU memory allocated by mem.

\subsection{Adversarial robustness examples}
\label{sec:timing_adversarial}
As in the examples studied in \Cref{sec:adv_robustness}, we perform measurements on a single minibatch consisting of 512 images from CIFAR-10. The proposed network initially performs 15 integration steps (corresponding to NonExpNet\textsubscript{init} in \Cref{tab:adv_timing}). During training with the adaptive approach, the maximum number of integration steps taken by the proposed network is 63 (corresponding to NonExpNet\textsubscript{max} in \Cref{tab:adv_timing}).
\begin{table}[!ht]
    \centering
    \caption{Timings and memory measurements for the models compared in \Cref{sec:adv_robustness}}
    \begin{tabular}{c|ccc|cc}
        \toprule
        \multirow{2}{*}{Method}       & \multicolumn{3}{|c}{Training} & \multicolumn{2}{|c}{Inference}                    \\
                                      &
        $t_\text{FB}$ (ms)            & $t_\text{P}$ (ms)             & mem (MB)                       &
        $t_\text{F}$ (ms)             & mem (MB)
        \\
        \midrule
        ResNet                        & 20                            & N/A                            & 1,758 & 6  & 896 \\
        ResNet-AT                     & 163                           & N/A                            & 1,782 & 6  & 896 \\\midrule
        NonExpNet\textsubscript{init} & 20                            & 5                              & 1,758 & 6  & 896 \\
        NonExpNet\textsubscript{max}  & 83                            & 5                              & 4,782 & 26 & 931 \\
        \bottomrule
    \end{tabular}
    \label{tab:adv_timing}
\end{table}

\subsection{Denoising examples}
\label{sec:timing_denoising}
As in the examples studied in \Cref{sec:denoising}, we perform measurements on a single minibatch consisting of 5 images from BSDS500. All of the proposed networks initially perform 10 integration steps (corresponding to Euler\textsubscript{init}, Heun\textsubscript{init} and RK4\textsubscript{init} in \Cref{tab:denoising_timing}). During training with the adaptive approach, the maximum number of integration steps taken by the proposed networks are 23, 22 and 26 steps for the networks using the Euler, Heun and RK4 integrator, respectively (shown in \Cref{tab:denoising_timing} as Euler\textsubscript{max}, Heun\textsubscript{max}, RK4\textsubscript{max}).

\begin{table}[!ht]
    \centering
    \caption{Timings and memory measurements for the models compared in \Cref{sec:denoising}}
    \begin{tabular}{c|ccc|cc}
        \toprule
        \multirow{2}{*}{Method}   & \multicolumn{3}{|c}{Training} & \multicolumn{2}{|c}{Inference}                          \\
                                  &
        $t_\text{FB}$ (ms)        & $t_\text{P}$ (ms)             & mem (MB)                       &
        $t_\text{F}$ (ms)         & mem (MB)
        \\
        \midrule
        TV                        & N/A                           & N/A                            & N/A    & 1,447 & 1,086 \\
        DnCNN                     & 143                           & N/A                            & 8,563  & 45    & 1,519 \\\midrule
        Euler\textsubscript{init} & 102                           & 8                              & 8,086  & 37    & 2,293 \\
        Euler\textsubscript{max}  & 236                           & 8                              & 15,612 & 84    & 2,485 \\\midrule
        Heun\textsubscript{init}  & 212                           & 8                              & 14,068 & 77    & 2,679 \\
        Heun\textsubscript{max}   & 466                           & 8                              & 27,964 & 168   & 2,870 \\\midrule
        RK4\textsubscript{init}   & 443                           & 8                              & 26,034 & 163   & 3,065 \\
        RK4\textsubscript{max}    & 1,155                         & 8                              & 63,090 & 423   & 3,450 \\
        \bottomrule
    \end{tabular}
    \label{tab:denoising_timing}
\end{table}

\end{document}